\documentclass[runningheads]{llncs}

\usepackage[T1]{fontenc}

\usepackage{xspace}
\usepackage{graphicx}
\usepackage{listings}
\usepackage{amsmath}
\usepackage{url}
\usepackage[disable]{todonotes}
\usepackage{hyperref}
\usepackage{footnote}
\usepackage{float}
\usepackage{tablefootnote}
\usepackage{subcaption}

\usepackage{listings}
\lstdefinestyle{shaclstyle}{
  basicstyle=\ttfamily\small,
  columns=fullflexible,
  keepspaces=true,
  breaklines=true,
}

\usepackage{url}

\usepackage{stmaryrd}

\usepackage{multicol}
\usepackage{multirow}
\usepackage{orcidlink} 

\usepackage{paralist}
\usepackage{enumitem}
\usepackage{graphicx}

\setlist[description]{leftmargin=20pt,labelindent=0pt,noitemsep,topsep=4pt,parsep=6pt,partopsep=5pt}

\newcommand{\ISA}{\sqsubseteq}

\makeatletter
\newcommand{\ostar}{\mathbin{\mathpalette\make@circled\star}}
\newcommand{\make@circled}[2]{%
  \ooalign{$\m@th#1\smallbigcirc{#1}$\cr\hidewidth$\m@th#1#2$\hidewidth\cr}%
}
\newcommand{\smallbigcirc}[1]{%
  \vcenter{\hbox{\scalebox{0.77778}{$\m@th#1\bigcirc$}}}%
}

\newcommand{\evalu}[2]{\llbracket #1 \rrbracket^{#2}}
\newcommand{\atleast}[3]{{\geq_{#1}\,}{#2}{.#3}}

\usepackage{amsmath}

\usepackage{mathtools} 
\usepackage{extarrows} 
\newcommand{\add}{\overset{\oplus}{\longleftarrow}}
\newcommand{\del}{\overset{\ominus}{\longleftarrow}}
\usepackage{tikz-cd}

\usepackage{hyperref}
\usepackage{url}
\begin{document}

\title{SHACL Validation under Graph Updates (Extended Paper)}
\titlerunning{SHACL Validation of Evolving Graphs}
\authorrunning{Shqiponja Ahmetaj et al.}
\author{Shqiponja Ahmetaj\inst{1}\orcidlink{0000-0003-3165-3568}
\and George Konstantinidis\inst{2}\orcidlink{0000-0002-3962-9303}
\and Magdalena Ortiz\inst{1}\orcidlink{0000-0002-2344-9658}
\and Paolo Pareti\inst{2}\orcidlink{0000-0002-2502-0011}
\and Mantas Simkus\inst{1}\orcidlink{0000-0003-0632-0294}
}
\institute{Vienna University of Technology (TU Wien), Austria \\
\email{shqiponja.ahmetaj@tuwien.ac.at} 
\and
University of Southampton, United Kingdom
}
\maketitle
\begin{abstract}
SHACL (SHApe Constraint Language) is a W3C standardized constraint language for RDF graphs. In this paper, we study SHACL validation in RDF graphs under updates. We present a SHACL-based update language that can capture intuitive and realistic modifications on RDF graphs and study the problem of static validation under such updates. This problem asks to verify whether every graph that validates a SHACL specification will still do so after applying a given update sequence. More importantly, it provides a basis for further services for reasoning about evolving RDF graphs.
Using a regression technique that embeds the update actions into SHACL constraints, we show that static validation under updates can be reduced to (un)satisfiability of constraints in (a minor extension of) SHACL. We analyze the computational complexity of the static validation problem for SHACL and some key fragments. Finally, we present a prototype implementation that performs static validation and other static analysis tasks on SHACL constraints and demonstrate its behavior through preliminary experiments.
\end{abstract}



\section{Introduction}

The SHACL (SHApe Constraint Language) standard, a W3C recommendation since 2017, provides a formal language for describing and validating integrity constraints on RDF data. 
In SHACL we can express, for example, that instances of the class Person must have exactly one date of birth.  
SHACL specifications consist of a \emph{shapes graph} that pairs \emph{shape} constraints with \emph{targets} specifying which nodes in a data graph must satisfy which shapes. 
The central service in SHACL is \emph{validation}, that is, determining whether an RDF graph conforms to a given shape graph. SHACL is being increasingly adopted as the basic means for ensuring data quality and integrity in RDF-based applications, and the development of SHACL validators and related tools is rapidly evolving \cite{KeZA24,Seifer0LS24,AhmetajBHHJGMMM25,OkulmusS24}.

RDF graphs can be large and may be subject to frequent changes and updates. 
The development of an appropriate standard for describing these updates is an ongoing effort in which the Semantic Web community has invested considerable effort \cite{Polleres:13:SU}.
While there are technologies for updating RDF graphs and validating them against SHACL specifications, these two steps currently need to be handled independently: if a validated graph is updated, it has to be re-validated from scratch, which can be expensive, and if an update leads to non-validation, returning to a valid state may be difficult or even impossible.  
This challenge becomes even more critical in privacy-sensitive domains and systems that integrate heterogeneous data sources, where direct access to the entire data graph is either restricted or impractical. For example, in healthcare or federated knowledge graphs, data may be held by external providers or evolve independently across sources, so executing an update that results in non-validation could be highly undesirable.
This motivates our study of \emph{validation under updates}, in which we study two problems of interest. The first is to determine whether a given update to a graph will preserve the validation of a SHACL constraint, enabling users to guarantee that data quality is preserved by identifying problematic updates before they are executed.
We then move to our main goal: determining whether a given update preserves certain SHACL constraints for every input data graph.


 Consider a hospital management system that stores patient data and enforces the SHACL constraints: (i) every patient must be linked to at least one address, and (ii) every address must include both a city and a house number.
This can be expressed in SHACL abstract syntax as the following shapes graph $(C,T)$
\begin{align*}
    C = & (\mathsf{PatientShape} \leftrightarrow \exists\,{hasAddress}.Address,\\
   & \mathsf{AddressShape} \leftrightarrow \exists\, {hasCity} \wedge \exists\,{hasHouseNumber}) \\
    T = &((Patient, \mathsf{PatientShape}), (Address, \mathsf{AddressShape}))
\end{align*}
Now, suppose that, as part of a privacy-preserving policy, the hospital decides to stop collecting the house numbers of patients' addresses. This is implemented as an update action that removes all $hasHouseNumber$ triples in RDF from addresses linked to patients, that is all atoms of the form  the $hasHouseNumber(x,y)$,\ where $x$ must satisfy the shape $\exists hasAddress^-.Patient$ and $y$ can be any node, and can be expressed in SHACL with $\top$. The question is whether such an update is safe: does it preserve the SHACL constraints for all input graphs? Unfortunately, in this case, the update violates
the constraint for $\mathsf{AddressShape}$, as the updated data would no longer contain required house numbers. 
This highlights a critical need for static validation; that is, before applying the update, we must verify that it will not break constraints for any valid input graph. Since data evolves frequently (e.g., hourly admissions), validating updates against every snapshot is infeasible. Static validation enables correctness guarantees across all current and future graphs, without exposing sensitive information. If an update is unsafe, one must revise either the update or the constraints (e.g., update the $\mathsf{AddressShape}$ to only require ${hasCity}$). 

In our setting, the modifications to RDF graphs are described in a simple update language that captures the core fragment of SPARQL Update. It is designed to be intuitive for SHACL practitioners, in the sense that SHACL expressions can be used to select the nodes affected by the updates and as preconditions in the actions.
We are not aware of other works on the preservation of SHACL constraints under updates. 
Some authors have examined the interaction between updates and schema information in RDF-S \cite{FlourisKAC13} or even ontological knowledge \cite{AhmetiCP14,WandjiC24,ZablithAdFKMPS15}.
In the context of graph databases, Bonifati et al. have considered the joint evolution of graphs and schema, focusing on different schema languages and reasoning problems \cite{BonifatiFGHOV19}. In relational databases, static validation is well-understood \cite{Kachniarz2001}. 
Particularly relevant to us is~\cite{AhmetajCOS17}, which studied validation under updates in dynamic graph-structured data, using a custom description logic \cite{DBLP:books/daglib/0041477} due to the lack of a standard constraint language. With SHACL now established as a W3C standard for RDF constraint specification, we revisit these problems in the SHACL setting. Our focus is to develop a foundational framework for reasoning about SHACL validation under RDF updates: we introduce a SHACL-aware action language, leverage SHACL validation for conditional execution, and study static validation under such updates. SHACL’s expressive language and specific validation semantics introduce technical challenges that require careful handling. A preliminary version of this work has been published in \cite{DBLP:conf/amw/Ahmetaj0S24}. The main contributions of this work are:

$\circ$ We introduce a declarative language for specifying updates on RDF graphs that uses SHACL shapes both to select affected nodes and to define preconditions for conditional updates. The language covers a large fragment of SPARQL Update and extends it with if–then–else constructs based on SHACL validation checks. To allow more expressive update actions, we also extend SHACL paths with features such as difference and more expressive targets.

$\circ$ We adapt the \emph{regression} method from \cite{AhmetajCOS17} to rewrite an input shapes graph by incorporating the effects of actions 'backwards', allowing us to show that validation under updates can be reduced to standard validation in a small extension of SHACL.  We prove the correctness of the method and show that the SHACL extension is necessary to capture the effects of actions on constraints accurately.      

$\circ$ We then study the problem of \emph{static validation under updates}, which checks whether the execution of a given action preserves the SHACL constraints for every initial data graph. Using the regression technique, we show that static validation under updates can be reduced to (un)satisfiability of a shapes graph in (a minor extension of) SHACL. Since satisfiability is known to be undecidable already for plain SHACL \cite{DBLP:conf/semweb/ParetiKMN20}, we leverage the results of \cite{AhmetajCOS17} to identify expressive fragments for which the problem is feasible in \textsc{coNexpTime} and \textsc{ExpTime}.  

$\circ$ We provide an implementation of static verification under updates 
for an expressive subset of the studied action language. We provide scalability results showing that, despite its high computational complexity, 
static verification is feasible 
for medium-sized shapes graphs and large numbers of actions.

\section{SHACL Validation}

In this section, we introduce RDF graphs and SHACL \emph{validation}. 
We follow the abstract syntax and semantics for the fragment of SHACL core studied in~\cite{shqi-etal-2021kr}; for more details on the W3C specification of SHACL core we refer to~\cite{shacl}, and for details of its relation with DLs to~\cite{DBLP:conf/wollic/Ortiz23,DBLP:conf/lpnmr/BogaertsJB22}.

\paragraph{\textbf{RDF Graphs.}} We let $N_N$, $N_C$, $N_P$ denote countably infinite, mutually disjoint sets of \emph{nodes} (constants), \emph{class names}, and \emph{property names}, respectively. 
An RDF (data) graph $G$ is a finite set of (ground) \emph{atoms} of the form $B({c})$ and $p({c},{d})$, where  $B \in N_C$, $p \in N_P$, and
${c},{d} \in N_N$. The set of nodes appearing in $G$ is denoted with $V(G)$. 
\paragraph{\textbf{Syntax of SHACL.}}
 Let $N_S$ be  a countably infinite set 
of \emph{shape names}, disjoint from $N_N$, $N_C$, $N_P$. A
\emph{shape atom} has the form  $\mathsf{s}(a)$, where
 $\mathsf{s} \in N_S$ and $a \in N_N$. A \emph{path expression} $E$ is a
regular expression built using the usual operators $*$, $\cdot$,
$\cup$ from symbols in
$N_P^{+}=N_P\cup \{p^{-}\mid p\in N_P \}$, where
$p^{-}$ is the \emph{inverse property} of $p$. A
\emph{(complex) shape} is an expression $\phi$ obeying~the~syntax:
\begin{align*}
 \phi,\phi'::= \top \mid \mathsf{s}\mid B \mid c \mid \phi\land \phi'  \mid \neg \phi \mid  \geq_n E.\phi \mid E = p  \mid disj(E,p) \mid closed(P)
\end{align*}
where  $\mathsf{s} \in N_S$, $p \in N_P$, $B \in N_C$, $c \in N_N$, $P \subseteq  N_P$, $n$ is a positive integer, and $E$ is a path expression. In what follows, we write
$\phi\lor \phi'$ instead of $\neg (\neg \phi\land \neg \phi')$;
$\geq_n E$ instead of $\geq_n E.\top$;
$\exists E.\phi$ instead of $\geq_1 E.\phi $; 
$\forall E.\phi$ instead of $\neg \exists E.\neg
\phi$. 

A \emph{(shape) constraint} is an expression of the form
 $\mathsf{s} \leftrightarrow \phi$, where  $\mathsf{s} \in N_S$ and $\phi$ is a possibly complex shape. \emph{Targets} in SHACL prescribe that certain nodes of the input data graph should validate certain shapes. A \emph{target expression} is of the form $(W,\mathsf{s})$, where  $\mathsf{s} \in N_S$, and $W$ takes one of the following forms: (i) constant from $N_N$, also called \emph{node target}, (ii) class name from $N_C$, also called \emph{class target}, (iii) expressions of the form $\exists p$ with $p \in N_P$, also called \emph{subjects-of target}, (iv) expressions of the form $\exists p^-$ with $p \in N_P$, also called \emph{objects-of target}.
A target is any set of target expressions. A \emph{shapes graph} is a pair $( {C,T})$, where ${C}$ is a set of constraints and ${T}$ is a set of targets. We assume that each shape name appearing in ${C}$ occurs exactly once on the left-hand side of a constraint, and each shape name occurring in $T$ must also appear in $C$. A set of constraints ${C}$ is \emph{recursive}, if there is a shape name in ${C}$ that directly or indirectly refers to itself. In this work, we focus on \emph{non-recursive} constraints. 

\paragraph{\textbf{Semantics of SHACL.}}
The evaluation of shape expressions is given by assigning nodes of the
data graph to (possibly multiple) shape names. More formally, a
\emph{(shape) assignment} for a data graph $G$ is a set $I = G \cup L$, where $L$ is a set of shape atoms such that ${a} \in V(G)$ for each ${s}({a}) \in L$. Let $V(I)$ denote the set of nodes that appear in $I$. The evaluation of a complex shape w.r.t.\,an assignment $I$ is given in terms of a function $\llbracket \cdot \rrbracket^I$ that maps a shape expression $\phi$ to a set of nodes, and a
path expression $E$ to a set of pairs of nodes as defined in Table~\ref{tab:evaluation}. We assume that the nodes appearing in a shapes graph $\mathcal{S}$ occur in the graph $G$.

 Let $I$ be a shapes assignment for a data graph $G$. We say $I$ is a \emph{model} of (or satisfies) a constraint $\mathsf{s} \leftrightarrow \phi$ if $\evalu{\mathsf{s}}{I} = \evalu{\phi}{I}$; $I$ is a model of a set of constraints $C$ if $I$ is a model of every constraint in $C$;  $I$ is a model of a target $(W,\mathsf{s})$ if $\evalu{W}{I} \subseteq \evalu{\mathsf{s}}{I}$; $I$ is a model of a target set $T$ if $I$ is a model of every $(W,\mathsf{s}) \in T$. Assume a SHACL shapes graph $( {C,T} )$ and a data graph $G$ such that each node that appears in ${C}$ or ${T}$ also appears in $G$. The data graph $G$ \emph{validates}
$( {C,T} )$ if there exists an assignment $I = G \cup L$
for $G$ such that 
\begin{inparaenum}
\item[(i)] $I$ is a model of ${C}$, and 
\item[(ii)] $I$ is a model of $T$.
\end{inparaenum}
 
 Clearly, for non-recursive constraints, which is the setting we consider here, the unique assignment obtained in a bottom-up fashion, starting from $G$ and evaluating each constraint once, is a model of $C$. More precisely, one can start from constraints of the form  $\mathsf{s}_1 \leftrightarrow \phi_1$, where $\phi_1$ has no shape names, and add to $G$ all the atoms  $\mathsf{s}_1(a)$ such that $a \in \llbracket \phi_1 \rrbracket^G$; let the result be $G \cup L_1$. We then proceed with constraints of the form  $\mathsf{s}_2 \leftrightarrow \phi_2$, where $\phi_2$ has only shape names occurring in $L_1$ and add in $L_2$ all  $\mathsf{s}_2(a)$ such that $a \in \llbracket \phi_2 \rrbracket^{G\cup L_1}$; the new assignment is $G \cup L_1 \cup L_2$. By iteratively evaluating all the  constraints in this manner we obtain in polynomial time a unique  
 assignment, denoted $I_{G,C} = G \cup L_{G,C}$
 that 
 is a model of $C$. 
 If $I_{G,C}$ is also a model of $T$, then  we that 
 $G$ \emph{validates}
$( C, T )$. 
We call a shapes graph $( C, T )$ \emph{satisfiable} if there exists some data graph $G$ that validates it. 
    \begin{table}
  \renewcommand{\arraystretch}{1.7}
  \begin{tabular}{l}
     $\evalu{\top}{I}=  V(I)$ \qquad \qquad \qquad  \quad $\evalu{{c}}{I} =  \{{c}\}$ \qquad \qquad \qquad  $\evalu{B}{I} =  \{{c} \mid B({c}) \in I\}$ 
   \\
    ${ \mathsf{s} }^{I} =  \{{c}\mid {\mathsf{s}}({c})\in I\}$ \qquad \qquad 
    $ \evalu{ \neg \phi}{I}=  V(I) \setminus \evalu{  \phi}{I}$  \qquad 
    $\evalu{ \phi_1\land \phi_2 }{I} =  \evalu{ \phi_1 }{I}\cap \evalu{ \phi_2 }{I}$ \\
        $\evalu{ p  }{I}=   \{ ({a},{b})\mid p({a},{b})\in I \}$ \qquad\qquad  
     $\evalu{ p^-  }{I}=  \{ ({a},{b})\mid p({b},{a})\in I \}$ \\
     $\evalu{ E\cdot E'   }{I}=  \{({a},{b}) \mid \exists d: (a,d) \in \evalu{ E}{I} \mbox{ and } (d,b) \in \evalu{E'}{I} \}$ \\
    $\evalu{ E\cup E'  }{I}=   \evalu{ E  }{I} \cup  \evalu{  E'  }{I}$ \\
    $\evalu{E^{*}  }{I}=  \{({a},{a}) \mid {a} \in V(I)\} \cup \evalu{ E}{I} \cup \evalu{E\cdot E}{I} \cup \cdots$ \\
    $\evalu{ \atleast{n}{E}{\phi} }{I}=  \{ {c} \mid |\{ ({c}, {d})  \in \evalu{E}{I} \text{ and } {d} \in \evalu{ \phi }{I} \}| \geq n \}$  \\
    $\evalu{ E = p }{I}=  \{ {c} \mid \forall {d}:({c},{d}) \in \evalu{E}{I} \mbox{ iff } ({c}, {d}) \in \evalu{p}{I} \}$   \\
    $\evalu{ disj(E,p) }{I}=  \{ {c} \mid \not\exists {d}:({c},{d}) \in \evalu{E}{I} \mbox{ and } ({c}, {d}) \in \evalu{p}{I} \}$  \\
    $\evalu{ closed(P) }{I}=  \{ {c} \mid \forall p \notin P:{c} \notin \evalu{\exists p}{I}\}$ 
  \end{tabular}
  \caption{Semantics of SHACL shape expressions  \label{tab:evaluation}}
\end{table} 

\vspace{-2\baselineskip}



\begin{example}\label{ex:validation} Consider the following data graph $G$ and shapes graph $(C,T)$: 
\begin{align*}
    G= \{ & Patient(p_1), Patient(p_2), ActivePatient(p_1), ActivePatient(p_2),  \\
          & Physician(Ann), Physician(Ben), Physician(Tom),  \\
     & treatsPatient(Ann, p_1), treatsPatient(Ben, p_1), treatsPatient(Tom,p_2) \} \\
     C =  \{& \mathsf{PatientShape} \leftrightarrow ActivePatient \lor DischargePatient, \\
     & \mathsf{PhysicianShape} \leftrightarrow Physician \lor \exists treatsPatient.ActivePatient \}, \\
      T = \{ & (Patient, \mathsf{PatientShape}), (\exists treatsPatient, \mathsf{PhysicianShape})\}
\end{align*}
The data graph describes two patients, $p_1$ and $p_2$, both marked as active patients. There are three physicians: $Ann$, $Ben$, and $Tom$. $Ann$ and $Ben$ are assigned to treat $p_1$, and $Tom$ treats $p_2$. The shapes graph specifies the following constraints.	Every node that is a patient must also be either an active patient or a discharged patient. Every node that treats a patient (i.e., has an outgoing $treatsPatient$-edge) must either be a physician or must treat a patient who is an active patient. Clearly, $G$ validates $(C,T)$ since $I_{G,C}$ is a model of $T$, where $I_{G,C}$ contains the shape atoms $\mathsf{PatientShape}(p_1)$, $ \mathsf{PatientShape}(p_2)$, $\mathsf{PhysicianShape}(Ann)$, $\mathsf{PhysicianShape}(Ben)$, $\mathsf{PhysicianShape}(Tom)$.
\end{example}

\section{Updating RDF Graphs}

In this section, we define an update language that is both expressive and tightly integrated with SHACL validation. To achieve this, we begin by extending SHACL itself. In our framework, SHACL shape graphs can serve as preconditions for actions, enabling validation checks prior to applying updates. To support richer and more practical update scenarios, we extend SHACL with richer targets, addressing a known limitation of the standard. In similar spirit, the SHACL Advanced Features Working Group~\footnote{\url{https://www.w3.org/TR/shacl-af/##targets}} is addressing this limitation by proposing SPARQL-based targets as a more flexible targeting mechanism. 

Moreover, we enrich path expressions by allowing constructs such as difference operators and some special type of properties, which are crucial for capturing fine-grained structural changes in RDF graphs. This extension is particularly essential for precisely capturing the impact of updates on SHACL constraints. While the expressivity required for such checks could in principle be captured by first-order logic, we opt to remain within an extended SHACL fragment, whose connection to Description Logics makes it well-suited for studying the complexity of static validation under updates for various fragments of SHACL.

\paragraph*{\textbf{Extending SHACL to Handle Updates.}}

We propose an extension of SHACL which we denote with SHACL$^+$. The syntax of SHACL$^+$ is defined as for SHACL, with the following extensions. First, we extend path expressions from SHACL to 
allow in addition to  the usual operators also a new \emph{difference} operator ($\setminus$) on symbols from $N^+_P$ and \emph{shape properties} of the form $(\phi_1,\phi_2)$, where $\phi_1$, $\phi_2$ are shape expressions without shape names. More specifically, path expressions in SHACL$^+$ are of the following syntax:
\begin{align*}
 E& ::= p \mid p^- \mid (\phi_1,\phi_2) \mid E \cdot E \mid  (E)^* \mid E \cup E \mid E \setminus E
\end{align*}
 where, $p, p^- \in N^+_P$, and $\phi_1$, $\phi_2$ are SHACL$^+$ shape expressions without shape names. Shape expressions are defined over such extended paths as expected. For example, $r \cdot (q \setminus (\exists r, \exists p))$ is an $E$-path expression; $(\exists r, \exists p)$ intuitively represents the Cartesian product of the domains of $r$ and $p$. Note that a \emph{singleton} property of the form $(a,b)$ is a special case where $a, b \in N_N$. 
Target expressions are of the form $(\phi, s)$, where $\phi$ is a complex shape without shape names. Targets are any Boolean combinations of target expressions. That is, if $T_1$ and $T_2$ are targets, then $T_1 \land T_2$, $T_1 \lor T_2$, and $\neg T_1$ are targets; we may write $T_1,T_2$ instead of $T_1 \land T_2$. The notions of shapes constraints and shapes graph are defined as for SHACL.

The evaluation of shape expressions w.r.t. an assignment $I$ is defined as in Table~\ref{tab:evaluation}, 
 and we add the  following two extensions for the evaluation of shape properties and path difference: 
\begin{align*}
    \evalu{(\phi_1,\phi_2)}{I} = & \{(a,b) \mid a \in \evalu{\phi_1}{I}, b \in \evalu{\phi_2}{I}\} \\ 
    \evalu{ E\setminus E'^+   }{I}= & \{({a},{b}) \mid (a,b) \in \evalu{ E}{I} \mbox{ and } (a,b) \not\in \evalu{E'^+}{I} \}.
\end{align*} That $I$ is a model of a target expression $(\phi, \mathsf{s})$ is naturally defined as $\evalu{\phi}{I} \subseteq \evalu{\mathsf{s}}{I}$. 
For targets, the definition of satisfaction is defined as expected, that is $I$ models $T_1$ and $T_2$ if $T$ is of the form $T_1 \land T_2$; $I$ models $T_1$ or $T_2$ if $T$ is of the form $T_1 \lor T_2$, and $I$ does not $T_1$, if $T$ is of the form $\neg T_1$. The notion of validation is then the same as for SHACL. 

\paragraph*{\textbf{SHACL-based Update Language for RDF Graphs}}
\label{sec:updates}

We now introduce an action language for updating RDF graphs under SHACL constraints, leveraging SHACL itself to define the syntax and semantics of the language. To allow for more flexible actions and to capture a wider range of updates, we base our language on the extension SHACL$^+$. The update language is composed of two types of actions, namely \emph{basic} and \emph{complex} actions. Basic actions enable two core operations: (i) adding or removing nodes satisfying a shape expression from the extension of a class, and (ii) adding or removing edges (properties) between pairs of nodes connected via a path expression. Complex actions allow for composing multiple actions and conditional executions based on the outcome of SHACL$^+$ validation checks over the graph.

To make actions more flexible, we assume a countable infinite set $N_V$ of variables disjoint from $N_N$, $N_C$, $N_P$, and $N_S$. In particular, we may allow variables in SHACL$^+$ shapes graphs in places of nodes. Path and shape expressions, targets, and shapes graphs with variables we call \emph{path formulas}, \emph{shape formulas}, \emph{target formulas} and \emph{shapes graph formulas}, respectively. 

\emph{Basic actions} $\beta$ and \emph{complex actions} $\alpha$ are defined by the following grammar:
\begin{align*}
\beta ::= {} & 
(B \add  \phi) \mid (B \del \phi) \mid (p \add E) \mid (p \del E) \\[1mm]
\alpha ::= {} &  \emptyset \mid \beta \cdot \alpha \mid (\mathcal{S}?\alpha [\alpha]) \cdot \alpha
\end{align*} 
where $B \in N_C$, $\phi$ is a SHACL$^+$ shape formula without shape names but possibly with variables,  
$p\in N_P$, $E$ is a  SHACL$^+$ path formula, 
$\emptyset$ denotes the empty action, and $\mathcal{S}$ is a SHACL$^+$ shapes graph formula.

A \emph{substitution} is a function $\sigma$ from $N_V$ to $N_N$. For any formula, an action, $\gamma$, we use $\sigma(\gamma)$ to denote the result of replacing in $\gamma$ every occurrence of a variable $x$ by a constant $\sigma(x)$.  An action $\alpha$ is ground if it has no variables, and $\alpha'$ is a ground instance of an action $\alpha$ if $\alpha' = \sigma(\alpha)$ for some substitution $\sigma$.
Intuitively, an application of a ground action $(A \add \phi)$ (or $(A \del C)$) on a graph $G$ stands for the addition (or deletion) of $A(c)$ to $G$ for each $c$ that makes $\phi$ true in $G$. 
Similarly, $ (p \add E)$ adds a $p$-edge whenever there is an $E$-path between two nodes; 
the case with deletion is as expected. Composition stands for successive action execution, and a conditional action $\mathcal{S}?\alpha_1 [\alpha_2]$ expresses that $\alpha_1$ is executed if $G$ validates $\mathcal{S}$, and $\alpha_2$ is performed otherwise. If $\alpha_2 = \emptyset$, then we have an action with a simple precondition as in classical planning languages, and write it $\mathcal{S}?\alpha_1$. 
The semantics of applying actions on graphs is defined only for ground actions. 

\begin{definition} Let  $G$ be a data graph and $\alpha$ a ground action. For a basic ground action $\beta$, $up(G,\beta)$ is defined as follows:
\begin{itemize}
    \item $up(G,\beta) = G \cup \{B(a) \mid a \in \evalu{\phi}{G}\}$ for $\beta = (B \add  \phi)$,
        \item $up(G,\beta) = G \cup \{p(a,b) \mid (a,b) \in \evalu{E}{G}\}$ for $\beta = (p \add E)$. 
\end{itemize}
The case with deletion $(\del)$ is analogous. 
Then, the result $up(G, \alpha)$ of applying a ground action $\alpha$ on $G$ is a graph defined recursively as follows: 
\begin{itemize}
     \item $up(G, \emptyset) = G$,
    \item $up(G, \beta \cdot \alpha') = up(up(G,\beta), \alpha)$ if $\alpha$ is of the form $\beta \cdot \alpha'$, and
    \item if $\alpha$ is of the form $ (\mathcal{S}?\alpha_1 [\alpha_2]) \cdot \alpha'$, then $up(G, \mathcal{S}?\alpha_1 [\alpha_2]) \cdot \alpha')$ is
    \begin{itemize}
        \item  $up(G, \alpha_1 \cdot \alpha')$, if $G$ validates $\mathcal{S}$, and
        \item  $up(G, \alpha_2 \cdot \alpha')$ if $G$ does not validate $\mathcal{S}$.
    \end{itemize}
\end{itemize}
\end{definition}
We illustrate the effects of an action update on a data graph with an example. 
\begin{example}\label{ex:projtermin} Consider $G$ and $(C,T)$ from Example~\ref{ex:validation}. Now, consider the action $\alpha$ that discharges patient $p_2$, which is removed from the active patients and added to the discharged ones; the physicians treating only this patient are removed. 
\begin{align*}
(ActivePatient \del p_2)\cdot (DischargePatient \add p_2)\cdot \\(Physician \del \forall treatsPatient.p_2)
\end{align*}
After applying $\alpha$ to $G$, we obtain the  graph $(G \cup \{DischargePatient(p_2))\} \setminus \{ActivePatient(p_2), Physician(Tom)\}$.
 The updated data graph $up(G,\alpha)$ does not validate the shapes graph $(C,T)$ since now $Tom$ will still have a $treatsPatient$-edge to patient $p_2$, but he does not satisfy the shape constraint for $\mathsf{PhysicianShape}$ as $Tom$ is not a physician and does not treat an active patient. 
 
\end{example}

We have not defined the semantics of actions with variables, that is non-ground actions. In our framework, variables are treated as parameters that must be instantiated with concrete nodes before execution. 

\begin{example} The action $\alpha_2$ with variables $x, y, z$ transfers the physician $x$ from treating patient $y$ to treating patient $z$ and is as follows:
\begin{align*}
& (\mathsf{s} \leftrightarrow Physician \land \exists treatsPatient.y, s' \leftrightarrow Patient), ((x, s), (y,s'), (z, s'))?   \\
            &(treatsPatient \del (x,y)) \cdot (treatsPatient \add (x,z))
\end{align*} 
Under the substitution $\sigma$ with $\sigma(x) = Tom$, $\sigma(y) = p_2$, $\sigma(z) = p_1$, the action $\alpha_2$ first checks whether the target node $Tom$ is an instance of the class $Physician$ and has a $treatsPatient$-property to $p_2$, and whether target nodes $p_1$ and $p_2$ are patients. If this is the case, the action removes the $treatsPatient$-edge between $Tom$ and $p_2$ and creates a $treatsPatient$-edge between $Tom$ and $p_1$.
\end{example}

In many scenarios, it is desirable for actions to have the ability to introduce “fresh” nodes into a data graph. Intuitively, the introduction of new nodes can be modeled in our setting by separating the domain of an assignment into the active domain and the inactive domain. The active domain consists of all nodes that occur in the data and shapes graph, whereas the inactive domain contains the remaining nodes. The inactive domain serves as a supply of fresh nodes that can be introduced into the active domain by executing actions. Since we focus on finite sequences of actions, a sufficiently large (but finite) inactive domain can always be assumed in the initial graph to provide an adequate supply of fresh constants. Likewise, node deletion can be captured by actions that move elements from the active domain back into the inactive domain.

\begin{example}\label{ex:preserving}
        Consider again our running example. Consider the action 
\[\alpha' =  (treatsPatient \del (\exists treatsPatient.p_2,p_2)),\] which removes the $treatsPatient$-edges to $p_2$ from physicians treating $p_2$. Clearly, $\alpha'$ can be applied to the updated graph $up(G,\alpha)$ from Example~\ref{ex:projtermin} by deleting $treatsPatient(Tom, p_2)$, resulting in a valid graph. Hence, the action $\alpha \cdot \alpha'$ is such that $up(G,\alpha \cdot \alpha')$ validates $(C,T)$. 
\end{example}

\section{Capturing Effects of Updates}\label{sec:alchoiq}

In this section, we define a transformation $tr_{\alpha}(\mathcal{S})$ that rewrites an input SHACL shapes graph $\mathcal{S}= (C,T)$ to capture all the effects of an action $\alpha$. This transformation can be seen as a form of regression, which incorporates the effects of a sequence of actions starting from the last to the first. 
More precisely, the transformation $tr_{\alpha}(\mathcal{S})$ takes a SHACL shapes graph $\mathcal{S}$ and an action $\alpha$ and rewrites them into a new shapes graph $\mathcal{S}_\alpha$, 
such that for \emph{every} data graph $G$: 
$up(G,\alpha) \mbox{ validates } \mathcal{S} \mbox{ iff } G \mbox{ validates } \mathcal{S}_\alpha$.
The idea is to simply update the constraints in $C$ and target expressions in $T$ accordingly, namely the actions replace every class name $B$ in $C$ and $T$ with $B \wedge C$ (or $B \land \neg C$) if the action is $B \add C$ (or $(B \del C)$); 
analogously for actions over properties.
If the action is of the form $\mathcal{S'}?\alpha_1 [\alpha_2]$, then we create two shapes graphs: one shapes graph $\mathcal{S}_{\alpha_1}$ if $G$ validates $\mathcal{S'}$ and $\mathcal{S}_{\alpha_2}$ if $G$ does not validate $\mathcal{S'}$. 

Before we proceed, to capture the effects of actions, we allow for Boolean combinations of shapes graphs. That is, if $\mathcal{S}$ and $\mathcal{S'}$ are shapes graphs, then $\mathcal{S} \land \mathcal{S'}$, $\mathcal{S} \lor \mathcal{S'}$, and $\neg \mathcal{S}$ are shapes graphs. Clearly, the notion of validation is defined as expected. However, allowing for boolean combinations of shapes graphs is just syntactic sugar, as they can be easily transformed (in linear time) into a single equivalent shapes graph $(C,T)$ in SHACL$^+$ that preserves validation (see Proposition~1 in Appendix).

\begin{definition}\label{def:transformation} Assume a SHACL$^+$ shapes graph $\mathcal{S}$ and a ground action $\alpha$. We use $\mathcal{S}_{Q \leftarrow Q'}$ to denote the new shapes graph that is obtained from $\mathcal{S}$ by replacing in $\mathcal{S}$ every class or property name $Q$ with the expression $Q'$. Then, the transformation $tr_\alpha({\mathcal{S}})$ of $\mathcal{S}$ w.r.t., $\alpha$ is defined recursively as follows:
\begin{align*}
    tr_{\epsilon}({\mathcal{S}}) &= {\mathcal{S}} \\
    tr_{(B \add C)\cdot \alpha}({\mathcal{S}}) &= (tr_{\alpha}({\mathcal{S}}))_{B \leftarrow B \lor C} \\
    tr_{(B \del C)\cdot \alpha}({\mathcal{S}}) &= (tr_{\alpha}({\mathcal{S}}))_{B \leftarrow B \land \neg C} \\
    tr_{(p \add E)\cdot \alpha}({\mathcal{S}}) &= (tr_{\alpha}({\mathcal{S}}))_{p \leftarrow p \cup E} \\
    tr_{(p \del E)\cdot \alpha}({\mathcal{S}}) &= (tr_{\alpha}({\mathcal{S}}))_{p \leftarrow p \setminus E} \\
    tr_{(\mathcal{S'}?\alpha_1 [\alpha_2]) \cdot \alpha}({\mathcal{S}}) &= (\neg\mathcal{S'} \lor tr_{\alpha_1\cdot \alpha}({\mathcal{S}})) \land (\mathcal{S'} \lor tr_{\alpha_2\cdot \alpha}({\mathcal{S}}))
\end{align*}
    \end{definition}

The transformation correctly captures the effects of complex actions. 

\begin{theorem}\label{th:regrescorrect}
    Given a ground action $\alpha$, a data graph $G$ and a SHACL shapes graph $\mathcal{S}$. Then, $up(G,\alpha)$ validates $\mathcal{S}$ if and only if $G$ validates $tr_{\alpha}(\mathcal{S})$.
\end{theorem}  
\begin{proof}[Sketch.] 
We prove the claim by induction on the structure of the action $\alpha$. 
For the base case, the claim trivially holds: $up(G,\emptyset) = G$ and $tr_\emptyset(\mathcal{S}) = \mathcal{S}$.  For the step, for actions $\beta \cdot \alpha'$, 
for each type of basic action $\beta$ it can be shown by induction on the structure of shape expressions that for every shapes graph $\mathcal{S'}$, it holds that $up(G,\beta)$ validates $\mathcal{S'}$ iff $G$ validates $\mathcal{S'}_{\beta}$. For conditional actions of the form $(\mathcal{S'}?\alpha_1 [\alpha_2]\cdot \alpha')$ we split the cases based on whether $G$ validates $\mathcal{S'}$. 
\end{proof}


We illustrate the transformation above with an example. 

\begin{example}\label{ex:transform} Consider $\mathcal{S}= (C,T)$ and $\alpha$ from our running example. Then, the transformation $tr_\alpha({\mathcal{S}})$ is the new shapes graph $(C',T)$, where:
\begin{align*}
  C' = \{ & \mathsf{PatientShape} \leftrightarrow (ActivePatient \land \neg p_2) \lor (DischargedPatient \lor p_2),  \\
 &  \mathsf{PhysicianShape} \leftrightarrow (Physician \land \exists treatsPatient.\neg p_2)\lor\\
  & \hspace{2.9cm} \exists treatsPatient.(ActivePatient \land \neg p_2) \}, \\
 T = \{ & (Patient, \mathsf{PatientShape}), (\exists treatsPatient, \mathsf{PhysicianShape})\}      
\end{align*}

\end{example}

We extended SHACL to support expressive update actions, but if we restrict to sequences of basic additions using standard shape and path expressions, and deletions for shapes, the transformation stays within (almost) plain SHACL. Although targets may allow arbitrary shape expressions, these can be rewritten using only standard SHACL constructs with the minor extension of allowing $\top$. As our focus is on static validation via satisfiability checking, and in absence of SHACL satisfiability tools, we rely on first-order-logic-based tools, and strict conformance to plain SHACL is not essential; see Appendix for details.

\section{Static Validation under Updates for Arbitrary Graphs}

In this section, we consider a stronger form of reasoning about updates. As  shown in the previous section validation under updates can be reduced to standard validation within a minor extension of SHACL. Building on this result, we now turn to the central task of this paper: \emph{static validation under updates}. Broadly speaking, this task asks whether a given sequence of update actions always preserves the validation of a shapes graph, independently of a concrete data graph.

\begin{definition}
    Let $\mathcal{S}$ be a shapes graph and let $\alpha$ be an action. Then, $\alpha$ is an $\mathcal{S}$-\emph{preserving} action if for every data graph $G$ and for every ground instance $\alpha'$ of $\alpha$, we have that $G$ validates $\mathcal{S}$ implies $up(G,\alpha')$ validates $\mathcal{S}$. The \emph{static validation under updates} problem is:
    \begin{itemize}
        \item[] Given an action $\alpha$ and a shapes graph $\mathcal{S}$, is $\alpha$ $\mathcal{S}$-preserving?
    \end{itemize}
    \end{definition}

Using the transformation from Definition~\ref{def:transformation}, we can reduce static validation under updates to unsatisfiability of shapes graphs: an action $\alpha$ is not $\mathcal{S}$-preserving if and only if there is some data graph $G$ and a ground instance $\alpha^*$ of $\alpha$ such that $G$ validates $\mathcal{S}$ and $G$ does not validate $tr_{\alpha^*}(\mathcal{S})$. In particular, if such a ground instance of $\alpha$ exists, then there exists a ground instance obtained by substituting the variables with nodes from the input or an arbitrary fresh node. 

\begin{theorem} \label{th:statver} Let $\alpha$ be a complex action with $n$ variables and $\mathcal{S}$ a SHACL$^+$ shapes graph. Let $\Gamma \subseteq N_N$ be the set of nodes appearing in $\mathcal{S}$ and $\alpha$ together with a set of $n$ fixed fresh nodes. Then,
\begin{enumerate}
    \item[(i)] $\alpha$ is not $\mathcal{S}$-preserving, if and only if
    \item[(ii)]  $\mathcal{S}  \land \neg tr_{\alpha^*}(\mathcal{S})$ is satisfiable for some ground instance $\alpha^*$ obtained by replacing each variable in $\alpha$ with a node from $\Gamma$.  
\end{enumerate}
\end{theorem}

\begin{proof}[Sketch] For (i) implies (ii), assume $\alpha$ is not $\mathcal{S}$-preserving. Then, by definition there exists a data graph and a ground instance $\alpha'$ of $\alpha$ such that if the graph validates $\mathcal{S}$, then the result of applying $\alpha'$ to the graph validates $\mathcal{S}$. Let $G$ be such a graph and let $\sigma$ be such a substitution of the form $x_1 \rightarrow a_1, \ldots, x_n \rightarrow a_n$ with $\sigma(\alpha) = \alpha'$. That is, $G$ validates $\mathcal{S}$ and $up(G,\alpha')$ does not validate $\mathcal{S}$. The latter together with Theorem \ref{th:regrescorrect} implies that $G$ does not validate $tr_{\alpha'}(\mathcal{S})$. Hence, $G$ validates $\mathcal{S}  \land \neg tr_{\alpha'}(\mathcal{S})$. 
Now, let $\sigma^*$ be a substitution obtained from $\sigma$ as follows: $x_i \rightarrow a_i$ if $a_i$ appears in $\mathcal{S}$ or $\alpha$, and $x_i \rightarrow c_i$ with $c_i$ is a fresh node in $\Gamma$ if $a_i$ does not appear in $\mathcal{S}$ or $\alpha$. Moreover, let $G^*$ be the graph obtained by replacing in $G$ all nodes $a_i$ not appearing in $\mathcal{S}$ or $\alpha$ with $c_i$ from $\Gamma$. Clearly, $\alpha^* = \sigma^*(\alpha)$ is a desired ground instance of $\alpha$ and $G^*$ is such that $G^*$ validates $\mathcal{S}  \land \neg tr_{\alpha^*}(\mathcal{S})$. 
For (ii) implies (i), assume that  $\mathcal{S}  \land \neg tr_{\alpha^*}(\mathcal{S})$ is satisfiable. Let $\alpha^*$ be a ground instance of $\alpha$ obtained by replacing each variable in $\alpha$ with a node from $\Gamma$ and let $G$ be a graph that validates $\mathcal{S}  \land \neg tr_{\alpha^*}(\mathcal{S})$. Hence, $G$ validates $\mathcal{S}$ and $G$ does not validate $tr_{\alpha^*}(\mathcal{S})$. By Theorem \ref{th:regrescorrect}, the latter implies that $up(G,\alpha^*)$ does not validate $\mathcal{S}$, and therefore $\alpha$ is not $\mathcal{S}$-preserving.
\end{proof}
We illustrate static validation under updates with an example. 
\begin{example}
    The action $\alpha$ from Example \ref{ex:projtermin} is not $(C,T)$-preserving since $G$ validates $(C,T)$, but $up(G,\alpha)$ does not validate $(C,T)$. From Example \ref{ex:transform}, we can see that $G$ does not validate $(C',T)= tr_\alpha(C,T)$ since the model $I_{G,C'}$ does not satisfy $T$. That is, $Tom \in \evalu{\exists treatsPatient}{I_{G,C'}}$ but $Tom \not\in \evalu{\mathsf{PhysicianShape}}{I_{G,C'}}$ since $Tom$ does not satisfy any of the conjuncts defining the shape $\mathsf{PhysicianShape}$, and in particular $Tom \not\in \evalu{Physician \land  \exists treatsPatient. \neg p_2}{I_{G,C'}}$. Intuitively, nodes removed from the class $Physician$ should also be removed from the $treatsPatient$ property as in the $\mathcal{S}$-preserving action $\alpha^* =\alpha \cdot \alpha'$ from Example \ref{ex:preserving}. By applying $\alpha^*$ to $\mathcal{S}$, we obtain the following transformed shapes graph $tr_{\alpha^*}(\mathcal{S})$:
\begin{align*}   C^* = \{ & \mathsf{PatientShape} \leftrightarrow (ActivePatient \land \neg p_2) \lor (DischargedPatient \lor p_2),  \\
 &  \mathsf{PhysicianShape} \leftrightarrow (Physician \land  \exists E.\neg p_2)\lor \exists E.(ActivePatient \land \neg p_2) \}, \\
 T^* = \{ & (Patient, \mathsf{PatientShape}), (\exists E, \mathsf{PhysicianShape})\}      
\end{align*}
where $E$ is the path expression $(treatsPatient\setminus (\exists treatsPatient.p_2,p_2))$.
 \end{example}
\todo{to check whether the following is correct and clear}The above theorem provides an algorithm for static validation under updates by converting it into satisfiability checking of shapes graphs in SHACL$^+$. 
While satisfiability has been shown to be undecidable already for plain SHACL in~\cite{PARETI2022100721}, 
better upper bounds can be obtained by restricting SHACL$^+$ to Description Logic (DL)-like fragments. 
Specifically, in addition to disallowing expressions of the form $E = p$, $disj(E,p)$, and $closed(P)$, we also disallow path operators $ *$ and composition $\cdot$ in $E$, and limit shape properties in paths to singleton properties of the form $(a,b)$ where $a,b \in N_N$. For this fragment, the co-problem of static validation can be reduced to finite satisfiability in the description logic
$\mathcal{ALCHOIQ}^{br}$ -- an extension of the DL $\mathcal{ALCHOIQ}$ -- which has been shown in~\cite{AhmetajCOS17} to be \textsc{NExpTime}-complete. If we further restrict, shape expressions of the form $\geq_n E.\phi$ to only allow $n=1$, we can reduce to $\mathcal{ALCHOI}^{br}$, for which finite satisfiability is in \textsc{ExpTime}. 
The co\textsc{NExpTime}-membership holds even when allowing unrestricted shape properties which can be easily captured by extending $\mathcal{ALCHOIQ}^{br}$ with pairs 
of arbitrary concepts in places of roles. The matching upper bound immediately follows by extending with this construct the translation of $\mathcal{ALCHOIQ}^{br}$ \cite{AhmetajCOS17} to $\mathcal{C}^2$, the two-variable fragment of first-order predicate logic extended with counting quantifiers, where such constructs translate naturally. 
For the lower bounds, we reduce the \textsc{NexpTime} problem of finite satisfiability of $\mathcal{ALCOIQ}$ \cite{DBLP:journals/jair/Tobies00,AhmetajCOS17} (and  \textsc{ExpTime} \cite{AhmetajCOS17} of $\mathcal{ALCHOI}$) into the co-problem of static validation for SHACL under updates. 

Condition (ii) of Theorem \ref{th:statver} requires the existence of a ground instance among potentially exponentially many options, proportional to the number of variables appearing in $\alpha$. 
However,  
each shapes graph $\mathcal{S}$ can be straightforwardly translated into an equisatisfiable DL knowledge base $K_{\mathcal{S}}$ in the extended DLs $\mathcal{ALCHOI(Q)}^{br}$, which allows for boolean combinations of KBs. Exploiting this,  we can show the following: $\alpha$ is not $\mathcal{S}$-preserving 
iff $\mathcal{S} \land \neg tr_{\alpha^*}(\mathcal{S})$ is satisfiable for some ground instance $\alpha^* \in \Sigma$ iff $K_{\mathcal{S}} \land \neg K_{tr_{\alpha^c}(\mathcal{S})}$ is finitely satisfiable, where $\alpha^c$ is obtained from $\alpha$ by replacing each variable with a fresh node not occurring in $\mathcal{S}$ and $\alpha$.
Note that for a sequence of basic ground actions  $\alpha^*$,  the size of $tr_{\alpha^*}(\mathcal{S})$ may grow exponentially in the number of actions. Intuitively, if the same class or property name $B$ is updated $n$ times in the action, and each update action introduces another occurrence of $B$, then the resulting constraint may contain up to $2^n$ occurrences of $B$, leading to exponential blow-up. However, for static validation, it suffices to construct an equisatisfiable description logic knowledge base (or first-order logic formula), where we can control the size more effectively. To this end, we define a transformation $tr_{\alpha^*}(K_S)$, which introduces a fresh class or property name for each update action, along with simple equivalence axioms (e.g., $B' \equiv B \lor \varphi$) storing the effect of the action. 
The result of this transformation $tr_{\alpha^*}(K_S)$ is equisatisfiable to $K_{tr_{\alpha^*}}(\mathcal{S})$ but the size is linear in the input action if $\alpha^*$ is a sequence of basic actions. Clearly, for conditional actions, the transformation may yield a Boolean combination of exponentially many knowledge bases, but each of them will be of linear size. From the above follows that 
$K_{\mathcal{S}} \land \neg K_{tr_{\alpha^c}(\mathcal{S})}$ is finitely satisfiable iff $K_{\mathcal{S}} \land \neg tr_{\alpha^c}(K_{\mathcal{S}})$ is finitely satisfiable. It was shown in \cite{AhmetajCOS17} that finite satisfiability of $K_{\mathcal{S}} \land \neg tr_{\alpha^c}(K_{\mathcal{S}})$  KBs is in \textsc{NExpTime} for $\mathcal{ALCHOIQ}^{br}$ and \textsc{ExpTime} for $\mathcal{ALCHOI}^{br}$. The proof holds also for the extension of $\mathcal{ALCHOIQ}^{br}$ that allows shape properties in places of roles. 

The correctness of the following theorem immediately follows. 
\begin{theorem} \label{th:statvercomp}
We obtain the following complexity results for static validation:
\begin{itemize}
    \item The problem is \emph{undecidable}. It remains undecidable also when the input shapes graph and action uses only plain SHACL. 
    \item The problem is co\textsc{NExpTime}-complete, when the input uses the fragment of SHACL$^+$ that does not allow the operators $ *$ and $\cdot$ in path expressions, and expressions of the form $E = p$, $disj(E,p)$, and $closed(P)$. 
    \item The problem is \textsc{ExpTime}-complete if additionally shape expressions of the form $\geq_n E.\phi$ are restricted to only $n=1$ and shape properties in paths are restricted to singleton properties. 
\end{itemize}
    \end{theorem}

\section{Implementation and Experiments}

To implement a solver for the static validation under updates problem, given a complex action $\alpha$ and a shapes graph $\mathcal{S}$, we compute $\neg tr_{\alpha^*}(\mathcal{S})$ and then check whether $\mathcal{S}$ is contained in $tr_{\alpha^*}(\mathcal{S})$, as detailed in Theorem \ref{th:statver}. Our approach is based on the translation of SHACL shape graphs into equisatisfiable FOL sentences \cite{PARETI2022100721}. The SHACL2FOL tool \cite{pareti2024shacl2fol} performs this translation, generating such sentences in the TPTP \cite{sutcliffe1998tptp} format. The TPTP file is then used by a theorem prover to determine satisfiability, enabling the SHACL2FOL tool to check both the satisfiability of individual shape graphs, and the containment problem between two shape graphs.

To enable static validation under updates, we have extended SHACL2FOL to compute the transformation of a shape graph under updates. This functionality accepts a shapes graph and a list of actions, in JSON format, as inputs. Given a shape graph $\mathcal{S}$, and its equisatisfiable FOL translation $\varphi$, and a complex action $\alpha$, we have implemented the regression approach that generates a FOL translation $\varphi_{\alpha}$ equisatisfiable to $tr_{\alpha^*}(\mathcal{S})$. Our implementation currently supports actions in the form $p \add E$ and $p \del E$, where $E$ is either a pair of shapes $(\phi_1,\phi_2)$, or a SHACL property path. 
This implementation effectively follows the regression approach from Section \ref{sec:alchoiq}, but as a direct transformation of the FOL sentence $\varphi$. The formalisation of actions in JSON format, their translation into logic expressions, and their integration into a TPTP file are novel extensions of the SHACL2FOL tool. The code for this implementation is available through the link in the Supplemental Material Statement.

\begin{figure}[htbp]
\centering

\begin{minipage}[t]{0.24\textwidth} 
\lstset{style=shaclstyle}
\begin{lstlisting}
sh:path p;
sh:(equals/disjoint) r;
\end{lstlisting}
\end{minipage}%
\hfill
\begin{minipage}[t]{0.41\textwidth} 
\lstset{style=shaclstyle}
\begin{lstlisting}
sh:path p;
sh:qualified(Min/Max)Count i;
sh:qualifiedValueShape [
  sh:(hasValue/class) c; ];
\end{lstlisting}
\end{minipage}%
\hfill
\begin{minipage}[t]{0.30\textwidth} 
\lstset{style=shaclstyle}
\begin{lstlisting}
sh:path P;
sh:(hasValue/class) c;
\end{lstlisting}
\end{minipage}
\caption{Three constraint templates where ``/'' denotes an alternative, p and r are property names, P is property path, i is an integer and c is a constant.}
\label{fig:shacl-snippets}
\end{figure}
A comprehensive benchmark of real world usages of SHACL is still missing from the literature, and for this reason our evaluation focuses on studying the effects of updates on different types of SHACL constraints. In particular, we focused on the constraint components that have been identified as the most problematic for the problem of satisfiability checking, namely: property pair equality, the qualified cardinality constraints, and sequence, alternative and transitive paths \cite{PARETI2022100721}. We use those constraints both to define the initial shape graph, and the updates to be performed. Our aim is not to cover all the possible interactions between constraints, but to show scalability across a wide range of constraints known for their theoretical complexity.

To generate a synthetic test case we create a shape graph consisting of a number of shapes, each one having a single target, equally divided among the four SHACL target definitions (excluding the implicit one). Each shape is initialized with a random constraint chosen from property pair equality, disjointness, a cardinality constraint (with limits from 0 to 2), or a constraint over a property path, using the constraint templates shown in Figure \ref{fig:shacl-snippets}. Whenever a node needs to be initialized, it is chosen randomly from a predefined set of constants of size double that of the number of shapes. Relations are initialized randomly from a predefined set of constants in the same way, with the exception that this set always includes the class membership relation \texttt{rdf:type}. This ratio of constants to shapes was chosen empirically to enable actions and shapes to interact through the reuse of the same constants, while also allowing for a limited number of new constants to be introduced. This is also used to create test cases where the split between true and false results of the static validation are more evenly balanced. 
Actions are created randomly in one of the two allowed types. Constraints and paths for actions are created in the same way as the constraints for the shapes, and draw from the same sets of node and relation constants.

We run our experiments using an Intel Core i7-9750H CPU and each datapoint shows the average over 10 randomized runs. We use the Vampire 4.9 theorem prover to determine the satisfiability of the TPTP sentences, requiring finite models for satisfiability, and using the \texttt{--mode portfolio} option to enable diverse, pre-tuned strategies to increase the likelihood of quickly solving the problem. In our first experiment, we use a shape graph of fixed size 10, and we measure the time to perform the static verification test as the number of actions scales from 1 to 200. The results of this experiments, in the left plot of Figure \ref{fig:increase}, show how the computation increases linearly with the number of actions. On average, the actions were shape preserving in 10\% of the cases. 

In our second experiment, we use an update list of 20 actions, and we scale the number of shapes in the original shapes graph from 10 to 70. The results of this second experiment, shown in the right plot of Figure \ref{fig:increase}, show how the computation time increases exponentially with the number of shapes. This exponential increase is in line with the known complexity of the shape containment problem \cite{PARETI2022100721}.  On average, the actions were shape preserving in 33\% of the cases.

Across both experiments, the theorem prover failed to find a proof in 16\% of the cases in the first experiment, and 8\% in the second. This is to be expected, as the satisfiability of certain combinations of shape constraints is known to be undecidable. It should also be noted that the performance of our tool strongly depends on the specific theorem prover being used and on its configuration. For example, Figure \ref{fig:increase} shows how enabling finite model checking increases the computation time.
Overall, the experiments show that the problem of static validation under updates, although complex, can be solved quickly for medium-sized shape graphs. It is important to note that the size of a shapes graph is linked with schema complexity and not data complexity, and thus it typically does not need to scale to very large numbers. The fact that the increase in computational time is linear in the number of actions is promising, showing that our regression approach could be applicable to domains involving large updates.

\begin{figure}[tbp]
  \centering
  \begin{subfigure}[b]{0.5\textwidth}
    \includegraphics[width=\linewidth]{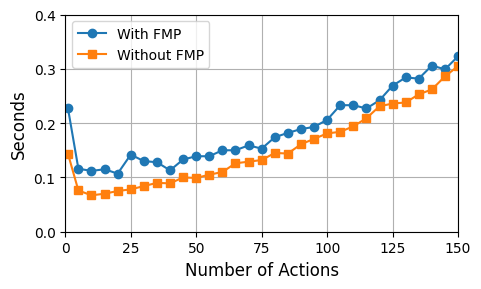}
    \label{fig:actions}
  \end{subfigure}
  \begin{subfigure}[b]{0.48\textwidth}
    \includegraphics[width=\linewidth]{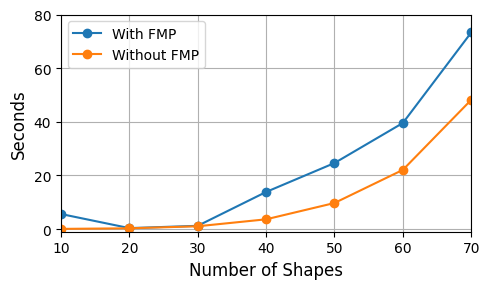}
    \label{fig:shapes}
  \end{subfigure}
  
  \caption{Time to solve the static validation problem with 10 shapes and an increasing number of actions (left), and with 20 actions and an increasing number of shapes (right), depending on usage of Finite Model Property.}
  \label{fig:increase}
\end{figure}

\section{Conclusion and Future Work} \todo{to rewrite}
We have presented a framework for formalizing updates for RDF graphs in the presence of SHACL constraints. We identified a suitable SHACL-aware update language, that can capture intuitive and realistic modifications on RDF graphs, covers a significant fragment of SPARQL updates and extends them to allow for conditional updates. We addressed an important problem: static validation under updates, which asks whether for every RDF graph that validates a SHACL shapes graph, the graph will remain valid after applying a set of updates. 
We showed that the latter problem can be reduced to (un)satisfiability of a shapes graph in a mild extension of SHACL and studied the complexity. 
We tested a prototype implementation of our regression-based method for static validation   
showing its potential 
for handling inputs of realistic size.

Toward lowering the complexity of static validation, we plan to analyze the problem for other relevant fragments of SHACL. We aim to further develop our implementation to support more action types, with the goal of creating a comprehensive tool for SHACL static analysis tasks.
We will also study other basic static analysis problems such as the {static type checking} problem \cite{DBLP:conf/pods/BonevaGHMS23}, which intuitively checks whether actions preserve a validation from source to target shapes graphs and problems related to planning.

\section*{Acknowledgements}

Ahmetaj was supported by the Austrian Science
Fund (FWF) projects netidee SCIENCE [T1349-N] and the Vienna Science and Technology Fund (WWTF) [10.47379/ICT2201]. Ortiz and Simkus were supported by the FWF projects PIN8884924, P30873 and
10.55776/COE12. The contributions of Pareti and Konstantinidis were supported by the UKRI Horizon Europe guarantee funding scheme for the Horizon Europe projects UPCAST (10.3030/101093216), RAISE (10.3030/101058479) and DataPACT (10.3030/101189771).

 \todo{we have 16 pages (excluding references)}
 \paragraph*{Supplemental Material Statement:} The code for our implementation of the regression approach, and its evaluation, is available on Zenodo \cite{paolo_2025_16633178}.\footnote{The code is also available at \url{https://github.com/paolo7/SHACL2FOL}} 

 \bibliographystyle{splncs04}
\bibliography{main}
\newpage
\section*{Appendix}

\paragraph{Boolean combination of shapes graphs}
For simplicity of presentation when capturing the effects of actions, in SHACL$^+$ we also allow for Boolean combinations of shapes graphs.
More formally, a SHACL$^+$ \emph{shapes graph} is defined recursively as follows: 
\begin{compactitem}
    \item  $(C,T)$ is a SHACL$^+$ shapes graph, and
    \item  if $\mathcal{S}_1$, $\mathcal{S}_2$ are SHACL$^+$ shapes graphs, then $\mathcal{S}_1 \land \mathcal{S}_2$, $\mathcal{S}_1 \lor \mathcal{S}_2$, and $\neg \mathcal{S}_1$ are SHACL$^+$ shapes graphs. 
\end{compactitem}
A SHACL$^+$ shapes graph is called \emph{normal}, if it is of the form $(C,T)$, where $C$ and $T$ are SHACL$^+$ constraints and targets, respectively. 
Validation is naturally defined as follows. 
\begin{definition}\label{def:boolshapesgraphs}
Consider SHACL$^+$ shapes graph $\mathcal{S}$ and data graph $G$. Then, $G$ \emph{validates} $\mathcal{S}$, if the following hold: 
\begin{compactitem}
    \item $I_{G,C} \models T$,  if $\mathcal{S}$ is of the form $(C,T)$, \item $G$ validates $\mathcal{S}_1$ and $\mathcal{S}_2$ if $\mathcal{S}$ if of the form $\mathcal{S}_1 \land \mathcal{S}_2$,
    \item $G$ validates $\mathcal{S}_1$ or $\mathcal{S}_2$ if $\mathcal{S}$ if of the form $\mathcal{S}_1 \lor \mathcal{S}_2$, and
    \item $G$ does not validate $\mathcal{S}_1$ if $\mathcal{S}$ is of the form $\neg \mathcal{S}_1$
\end{compactitem}

\end{definition}

Allowing Boolean combinations of shapes graphs is just syntactic sugar, as each 
SHACL$^+$ shapes graphs $\mathcal{S}$ can be converted into equivalent normal shapes graphs 
by simply renaming the shape names in each normal shapes graph appearing in $\mathcal{S}$, taking the union of all constraints, and pushing the Boolean operators to the targets. 
The following proposition immediately holds. 
\begin{proposition}\label{pr:boolshg} Consider a  
SHACL$^+$ shapes graph $\mathcal{S}$  and a data graph $G$. Then, $\mathcal{S}$ can be converted in linear time into a normal SHACL$^+$ shapes graph $(C_{\mathcal{S}}, T_{\mathcal{S}})$ such that
$G$ validates $\mathcal{S}$ iff $G$ validates $(C_{\mathcal{S}}, T_{\mathcal{S}})$, for every data graph $G$.
\end{proposition}

\vspace{1cm}
\noindent {\textbf{Theorem \ref{th:regrescorrect}}}
    Given a ground action $\alpha$, a data graph $G$ and a SHACL shapes graph $\mathcal{S}$. Then, $up(G,\alpha)$ validates $\mathcal{S}$ if and only if $G$ validates $tr_{\alpha}(\mathcal{S})$.
\begin{proof}[Sketch.] 
We prove the claim by induction on the size $l(\alpha)$ of the action $\alpha$. That is, $l(\emptyset) = 0$, $l(\beta \cdot \alpha) = 1 + l(\alpha)$, and $l(\mathcal{S}?\alpha_1 [\alpha_2]\cdot \alpha_3) = 1 + l(\alpha_1)+ l(\alpha_2)+ l(\alpha_3)$. For the base case, that is when $l(\emptyset) = 0$, the claim trivially holds by definition since $up(G,\emptyset) = G$ and $tr_\emptyset(\mathcal{S}) = \mathcal{S}$. 

For the induction step, assume that $\alpha$ is of the form $(B \add \phi) \cdot \alpha'$. Let $G' = G^{(B \add \phi)} = G \cup \{B(a) \mid a \in \evalu{\phi}{G}\}$, that is $G'$ is $G$ expanded with atoms $B(a)$  over nodes $a$ that satisfy $\phi$ in $G$. For every shapes graph $\mathcal{S'}$, it holds that $G'$ validates $\mathcal{S'}$ iff $G$ validates $\mathcal{S'}_{B \leftarrow B \lor \phi}$. This claim can be shown by an induction on the structure of the shapes graph $\mathcal{S'}$, and specifically on the structure of the shape expressions where $B$ occurs. In particular, $G'$ validates $tr_{\alpha'}(\mathcal{S})$ iff $G$ validates $(tr_{\alpha'}(\mathcal{S}))_{B \leftarrow B \lor \phi}$. Since,  $(tr_{\alpha'}(\mathcal{S}))_{B \leftarrow B \lor \phi} = tr_{\alpha}(\mathcal{S})$, we get $G'$ validates $tr_{\alpha'}(\mathcal{S})$ iff $G$ validates $tr_{\alpha}(\mathcal{S})$. By the induction hypothesis, we have that $G'$ validates $tr_{\alpha'}(\mathcal{S})$ iff $up(G', \alpha')$ validates $\mathcal{S}$. Hence,  $G$ validates $tr_{\alpha}(\mathcal{S})$ iff $G'$ validates $tr_{\alpha'}(\mathcal{S})$. Since $up(G', \alpha') = up(G, (B \add \phi)\cdot \alpha') = up(G,\alpha)$ by definition, we have the desired claim. 
Using similar arguments, we can show the other cases for $\alpha$ of the form $(B \del \phi) \cdot \alpha'$, or of the form $(p \add E)\cdot \alpha'$, and of the form $(p \del E) \cdot \alpha'$.

It is left to show the claim for the case with conditional actions  $\alpha$ of the form $(\mathcal{S'}?\alpha_1 [\alpha_2]\cdot \alpha')$. Assume $G$ validates $\mathcal{S'}$; the case where $G$ does not validate $\mathcal{S'}$ is analogous. By definition, $up(G, \alpha) = up(G, \alpha_1\cdot \alpha')$. By induction hypothesis, we know that $up(G, \alpha_1\cdot \alpha')$ validates $\mathcal{S}$ iff $G$ validates $tr_{\alpha_1\cdot \alpha'}(\mathcal{S})$ and hence, $up(G,\alpha)$ validates $\mathcal{S}$ iff $G$ validates $tr_{\alpha_1\cdot \alpha'}(\mathcal{S})$. Since $G$ validates $\mathcal{S'}$ and by definition of $tr_{(\mathcal{S'}?\alpha_1 [\alpha_2]\cdot \alpha')}(\mathcal{S})$ we can lift it up to $tr_\alpha(\mathcal{S})$. Hence, it follows that $up(G,\alpha)$ validates $\mathcal{S}$ iff $G$ validates $tr_{\alpha}(\mathcal{S})$. 
\end{proof}

\paragraph{\textbf{SHACL fragment}} We have extended SHACL to allow more expressive update actions to be performed over the data. However, if we restrict the actions to sequences of basic actions of the form: \[(B \add  \phi) \mid (B \del \phi) \mid (p \add E)\]
where $E$ and $\phi$ are SHACL path and shape expressions, respectively, then the result of the transformation would be in (almost) plain SHACL. Of course, the transformation could generate more complex expressions in targets of the form $(\phi, \mathsf{s})$, where $\phi$ is an arbitrary SHACL shapes expression without shape names. In fact, it suffices to simply allow targets expressions of the form $(\top, \mathsf{s})$. That is, for a shapes graph $(C,T)$ we can simply generate an equivalent shapes graph $(C',T')$, such that for each shape name $\mathsf{s}$ and for each target expression $(\phi_i, \mathsf{s}) \in T$ we add $(\top, \mathsf{s}) \in T'$ and for the corresponding constraint $\mathsf{s} \leftrightarrow \phi' \in C$ we add $\mathsf{s} \leftrightarrow \bigwedge_i(\neg \phi_i \lor \phi')$ in $C'$. 
It is important to note that update actions that delete $p$-edges between nodes connected via a path expression cannot be directly supported. The effects of such actions is captured through the operator ($\setminus$), which is not part of the standard SHACL path language.  

By working within this restricted action language, we are able to leverage existing SHACL validators for validation tasks under updates. While our approach tightly leverages SHACL validation to model updates, our focus is on static validation, and specifically on verifying the satisfiability of shapes graphs, rather than dynamic validation using SHACL tools. Therefore, given the current lack of tools for satisfiability checking in SHACL, the distinction between whether the resulting constraints strictly conform to SHACL or represent a minor extension is not central to our static analysis objectives. 


\vspace{1cm}
\noindent {\textbf{Theorem \ref{th:statvercomp}}}
We obtain the following complexity results for static validation:
\begin{itemize}
    \item The problem is \emph{undecidable}. It remains undecidable also when the input shapes graph and action uses only plain SHACL. 
    \item The problem is co\textsc{NExpTime}-complete, when the input shapes graph and action uses SHACL$^+$ constructs that do not allow: (1) the operators $ *$ and $\cdot$ in path expressions, and (2) shape expressions of the form $E = p$, $disj(E,p)$, and $closed(P)$. 
    \item The problem is \textsc{ExpTime}-complete if additionally shape expressions of the form $\geq_n E.\phi$ are restricted to only $n=1$ and shape properties in paths are restricted to singleton properties. 
\end{itemize}
\begin{proof}
For the hardness, it suffices to show that satisfiability of SHACL can be reduced in polynomial time to static verification of shapes graph expressed in plain SHACL. In particular, a SHACL shapes graph $\mathcal{S} = (C,T)$ is satisfiable if and only if the action $(B \add c)$ is not preserving the SHACL shapes graph $\mathcal{S'}= (C \cup \{s \leftrightarrow \neg B\}, T \cup \{(c, s)\})$, where $B \in N_C$, $s \in N_S$, $c \in N_N$ are fresh names not occurring in $(C,T)$. To show that $\mathcal{S}$ is satisfiable implies that $\alpha$ is not $\mathcal{S'}$-preserving, let $G$ be a graph that validates $(C,T)$ and does not use $B$ and $c$. Then, $G' = G \cup \{(B'(c))\}$, where $B'$ is a fresh class name, validates $\mathcal{S'}$, but $up(G', (B \add c))$ doesn't validate $\mathcal{S'}$ . For the other direction, let $G$ be a graph that validates $\mathcal{S'}$ such that $up(G, (B \add c))$ does not validate $\mathcal{S'}$. Since $G$ validates $\mathcal{S'}$ and since $B$, $s$, $c$ do not occur in $\mathcal{S}$, then $G$ also validates $\mathcal{S}$, which shows that $\mathcal{S}$ is satisfiable. 

For the subfragments of SHACL we consider, we reduce from finite satisfiability of description logic knowledge bases. First we introduce the description logic $\mathcal{ALCOIQ}$. An $\mathcal{ALCOIQ}$ knowledge base $K$ is a tuple consisting of an ABox and a TBox. W.l.o.g. we assume here that the ABox is empty. Then a TBox $K$ is a finite set of axioms of the form $\phi \sqsubseteq \phi'$, where $\phi$ and $\phi$ are (restricted) shapes expressions, called concepts in description logics, obeying the following grammar: 
    \begin{align*}
 \phi, \phi'::= \top \mid B \mid c \mid \phi\land \phi'  \mid \neg \phi \mid  \geq_n r.\phi 
 \end{align*} 
 where $r \in N^+_P$ is a role, $B \in N_C$, $c \in N_N$. The semantics is defined in terms of interpretations $I = (\Delta^I, \cdot^I)$. Since we are interested in finite satisfiability, we assume finite domain $\Delta^I$. Then, for each class name $B$, $B^I \subseteq \Delta^I$; for each property name $p$, $p^I \subseteq \Delta^I \times \Delta^I$, and $c^I \in \Delta^I$ for each node $c$. The function $\cdot^I$ is extended to all $\mathcal{ALCOIQ}$ expressions (namely concepts and roles) as usual (see Table \ref{tab:evaluation}). For an inclusion $\phi \sqsubseteq \phi'$, $I$ satisfies $\phi \sqsubseteq \phi'$, if $\phi^I  \subseteq \phi'^I$. The notion of satisfaction extends naturally to TBoxes and knowledge bases, that is a TBox $K$ is satisfiable if there exists an interpretation $I$ that satisfies all the axioms in the TBox. 
 Now let $K$ be an $\mathcal{ALCOIQ}$ TBox. For each concept inclusion $\phi \sqsubseteq \phi' \in K$, we write the shape expression $\neg \phi \lor \neg \phi'$. Let $C_K$ be the shape expression in conjunctive normal form obtained by taking the conjunction of all such inclusions in $K$. More precisely, $C_K =\{\bigwedge(\neg \phi \lor \neg \phi')\mid \phi \sqsubseteq \phi' \in K\}$. Then, $K$ is finitely satisfiable iff the action $(B \add \{c\})$ is not $(C ,T)$-preserving, where $C = (s \leftrightarrow {C}_K, s' \leftrightarrow \neg B')$, $T = \{(\top, s), (c,s')\}$, where $c \in N_N$ and $B \in N_C$ are not occurring in ${K}$. The correctness of this claim follows from analogous arguments to above. It was shown in  \cite{DBLP:journals/jair/Tobies00,AhmetajCOS17} that finite satisfiability of $\mathcal{ALCOIQ}$ is \textsc{NExpTime}-complete. From the claim follows the co\textsc{NExpTime}-hardness of static verification even for the case when the input shapes graph and action uses the fragment of plain SHACL that only allows for $\mathcal{ALCHOIQ}$ expressions, with the only addition of using $\top$ in the target. The same encoding can be used for $\mathcal{ALCOI}$ knowledge bases which only allow for $\exists r.\phi$ instead of $\geq_n r.\phi$, i.e., $n=1$, thus showing the \textsc{ExpTime} lower bound \cite{AhmetajCOS17}. 

 Membership for unrestricted SHACL$^+$ follows from ~\cite{DBLP:conf/semweb/ParetiKMN20}, which showed that checking satisfiability is undecidable already for (plain) SHACL shapes graphs. Obtaining a co\textsc{NExpTime} (and \textsc{ExpTime} upper bound) is more involved. To proceed we first show that every SHACL$^+$ shapes graph over the $\mathcal{ALCHOIQ}$ constructs described above and including shape names as well, can be converted into an equisatisfiable $\mathcal{ALCHOIQ}^{br}$ knowledge base -- we refer to \cite{AhmetajCOS17} for the full definition of this logic. Roughly speaking, this logic extends $\mathcal{ALCHOIQ}$ with the ($\setminus$) operation between properties, the singleton properties $(a,b)$ with $a,b \in N_N$, and boolean combinations of axioms and knowledge bases. 
 Note that 
 $\mathcal{ALCHOIQ}^{br}$ does not allow for arbitrary shape properties in places of roles, but an extension of this logic with such constructs preserves membership of finite satisfiability in \textsc{NExpTime}  -- hence showing membership in co\textsc{NExpTime} for static validation also with shape properties -- as the upper bound for finite satisfiability of $\mathcal{ALCHOIQ}^{br}$ follows from the translation into $\mathcal{C}^2$ \cite{AhmetajCOS17}, and the \textsc{NExpTime} upper bound for finite satisfiability of $\mathcal{C}^2$ formulas established by Pratt-Hartmann \cite{Pratt-Hartmann05} and encoding such expressions is natural in $\mathcal{C}^2$. In the following, by a slight abuse, when we write $\mathcal{ALCHOIQ}^{br}$ we consider the extended version that allows tuples $(\phi_1,\phi_2)$ of arbitrary $\mathcal{ALCHOIQ}^{br}$ concepts $\phi_1$, $\phi_2$ in places of roles; these are interpreted as expected, intuitively as the cartesian product of the nodes satisfying $C_1$ and $C_2$. Note that this does not hold for $\mathcal{ALCHOI}^{br}$, whose \textsc{ExpTime} membership of finite satisfiability is shown via translation to the guarded two-variable fragment of first-order logic and such constructs that compute the cartesian product would not be immediately translatable to this fragment. 
 
 Let $(C,T)$ be a SHACL$^+$ shapes graph that does not allow: (1) the operators $ *$ and $\cdot$ in path expressions, and (2) shape expressions of the form $E = p$, $disj(E,p)$, and $closed(P)$. Then $K_{C,T}$ is an $\mathcal{ALCHOIQ}^{br}$ KB defined as follows:
 \begin{itemize}
     \item We construct $K_C = \{\bigwedge(\mathsf{s} \sqsubseteq \phi) \mid \mathsf{s}\leftrightarrow \phi \in C\}$ as the conjunction of axioms of the form $\mathsf{s} \sqsubseteq \phi$  for each $\mathsf{s} \leftrightarrow \phi \in C$. 
     \item We construct $K_T$ over the target $T$ where instead of each  target expressions $(\phi, \mathsf{s})$ occurring in $T$ we write $\phi \sqsubseteq \mathsf{s}$. Note that in SHACL$^+$ we allow for boolean combinations of target expressions in $T$ and hence, $K_T$ will be a boolean knowledge base.  
 \end{itemize}
 Let $K_{C,T}$ be $K_C \land K_T$. Note that shape names in $C$ and $T$ are treated as class names. Then, it is not hard to see that following holds: \[(C,T) \mbox{ is satisfiable iff } K_{C,T} \mbox{ is satisfiable }(*)\] 

 Second, note that for sequences of basic actions the size of the constraints may grow exponentially in the number of actions. Consider for instance the constraint $\varphi = s \leftrightarrow B$ and the action sequence $\alpha_1 \cdot \alpha_2$ where $\alpha_1 = (B \add \exists r.B)$ and $\alpha_2 = (B \add \exists p.B)$. The transformation $tr_{\alpha_1 \cdot \alpha_2} (\varphi) = (tr_{\alpha_2}(\varphi))_{\alpha_1}$ on $\varphi$ first applies $\alpha_2$ by replacing $B$ replaced with $B \cup \exists  p.B$. Then, $\alpha_1$ is applied to the resulting constraint, namely $tr_{\alpha_1}(s \leftrightarrow B \lor \exists p.B)$ by replacing every occurrence of $B$ again with $B \lor \exists r.B$. The transformed  constraint is $s \leftrightarrow B \cup \exists r.B \cup \exists p.(B \lor \exists r.B)$, which now has 4 occurrences of $B$. Thus, for $n$ actions of the above form updating $B$, the number of occurrences of $B$ in the resulting constraint will be $2^n$. However, we can show that there is a DL KB of polynomial size (even linear) after the transformation w.r.t. an action that is satisfiable if and only if the shapes graph obtained after the transformation w.r.t. the action is satisfiable. 

 We first define the transformation of a knoweldge base $K$ w.r.t. a sequence $\alpha$ of ground actions defined over $\mathcal{ALCHOIQ}^{br}$. We use ${K}_{Q \leftarrow Q'}$ to denote the new KB that is obtained from $K$ by replacing in $K$ every class or property name $Q$ with the expression $Q'$. Then, the transformation $tr_\alpha(K)$ of $K$ w.r.t., $\alpha$ is defined recursively as follows:
\begin{align*}
    tr_{\epsilon}(K) &= {K} \\
    tr_{ (B \del C) \cdot \alpha}(K) &= (tr_{\alpha}(K))_{B \leftarrow B'} \land   B' \equiv B \land \neg \phi \\
     tr_{ (B \add C)\cdot \alpha}(K) &= (tr_{\alpha}(K))_{B \leftarrow B'} \land   B' \equiv B \lor \phi \\
      tr_{ (p \add E)\cdot \alpha}(K) &= (tr_{\alpha}(K))_{p \leftarrow p'} \land   p' \equiv p \cup E \\
       tr_{ (p \del E)\cdot \alpha}(K ) &= (tr_{\alpha}(K))_{p \leftarrow p'} \land   p' \equiv p \setminus E \\
    tr_{(K'?\alpha_1 [\alpha_2]) \cdot \alpha}({K}) &= (\neg K' \lor tr_{\alpha_1\cdot \alpha}({K})) \land (K' \lor tr_{\alpha_2\cdot \alpha}({K}))
\end{align*}
 where $B'$, $p'$ are fresh class and role names not occurring in the input $tr_{\alpha}(K)$.
It is straightforward to see that the result of the transformation w.r.t. a sequence of basic actions is a KB of linear size in the input. Intuitively, for a sequence $\alpha = \alpha_1 \cdot \alpha_2$ from above, the result of applying $\alpha$ to a KB $K_\phi = s \sqsubseteq B$ is the KB obtained by replacing each class name $B$ in $K$ with a fresh class name $B'$ and adding an equivalence axiom capturing the effect of the action, namely $tr_{\alpha_1}(K_{\phi}) = (K_{\phi})_{B \leftarrow B'} \land  (B' \equiv B \lor \exists p.B) = (s \sqsubseteq B') \land (B' \equiv B \lor \exists p.B)$. Now, $tr_\alpha(K_\phi) = (tr_{\alpha_1}(K_{\phi}))_{B \leftarrow B''} \land (B'' \equiv B \lor \exists r.B)$, which is the new KB $(s \sqsubseteq B') \land (B' \equiv B'' \lor \exists p.B'') \land (B'' \equiv B \lor \exists r.B)$. The algorithm applies the updates from right to left by introducing fresh names $B_1$, $B_2$, $\ldots$ that store intermediate updated versions of $B$. Substituting these back produces the same final constraint as applying updates directly backward. 

The above transformation produces a KB that grows lineraly in the size of the actions for basic actions.
Of course, action with preconditions can generate an exponential number of such KBs in the size of the action but each of them will be of linear size. 

Now, let $\alpha$ be a ground complex action that for an input and let $\mathcal{S}$ be a  shapes graph. Then, the following holds:

\[tr_{\alpha}(\mathcal{S}) \mbox{ is satisfiable iff }tr_{\alpha_{K}}(K_{\mathcal{S}}) \mbox{ is satisfiable.} (**)\]
where $\alpha_K$ is the action obtained from $\alpha$ by replacing each shapes graph $\mathcal{S}$ in $\alpha$ with $K_{\mathcal{S}}$. 
By claim (*), we have that $tr_{\alpha}(\mathcal{S}) $ is satisfiable iff $K_{tr_{\alpha}(\mathcal{S})}$ is satisfiable. It is left to show that $K_{tr_{\alpha}(\mathcal{S})}$ is satisfiable iff $tr_{\alpha_{K}}(K_{\mathcal{S}})$ is satisfiable, which can be shown by a straightforward induction on the structure of $\alpha$. Roughly, one can show that every model of $K_{tr_{\alpha}(\mathcal{S})}$ can be extended to a model of  $tr_{\alpha_{K}}(K_{\mathcal{S}})$ by simply interpreting the fresh class (property) names in the same way as the class (property) names they substituted. Similarly, every model of $tr_{\alpha_{K}}(K_{\mathcal{S}})$  restricted to the class and property names in $K_{\mathcal{S}}$ is a model of $K_{tr_{\alpha}(\mathcal{S})}$.
Unfolding the shape names with the definitions in the equivalence axioms in $tr_{\alpha_{K}}(K_{\mathcal{S}})$ results in the  knowledge base $K_{tr_{\alpha}(\mathcal{S})}$, which may be exponential in the size of $\alpha$. However, the equisatisfiable $tr_{\alpha_{K}}(K_{\mathcal{S}})$ is of size only linear in the input.

By condition (ii) of Theorem \ref{th:statver}, there exists a ground instance among potentially exponentially many such options in the number of variables appearing in $\alpha$. That is, for $n$ variables in $\alpha$ and $m$ nodes (constants) appearing in $\mathcal{S}$ and $\alpha$, the number of possible ground instances of the action $\alpha$ is bounded by $n^m + 1$. If neither $\mathcal{S}$ nor $\alpha$ contain nodes, then to obtain $\alpha^*$ it suffices to simply replace each variable in $\alpha$ with an arbitrary fresh node.
However, by exploiting the above claims, and the fact that in DLs the interpretation of two nodes may be the same domain element, we can show that we could simply consider one "canonical" ground instance of the action that sutbstitutes every variable with a fresh name not occurring in the input. We are ready for the main claim which reduces co-problem of static validation to satisfiability checking of an $\mathcal{ALCHOI(Q)}^{br}$ knowledge base. 

For an input shapes graph $\mathcal{S}$, and action $\alpha$ over $\mathcal{ALCHOI(Q)}^{br}$ constructs,  the following hold:

 \begin{enumerate}
    \item[(i)] An action $\alpha$ is not $\mathcal{S}$-preserving, if and only if
   \item[(ii)]  $\mathcal{S}  \land \neg tr_{\alpha^*}(\mathcal{S})$ is satisfiable for some ground instance $\alpha^*$ obtained by replacing each variable in $\alpha$ with a node from $\Gamma$, if and only if 
    \item[(iii)]    $K_{\mathcal{S}} \land \neg {tr_{\alpha^c_K}(K_{\mathcal{S}})}$ is finitely satisfiable,   where $\alpha^c_K$ is obtained from $\alpha_K$ by replacing each variable with a fresh node not occurring in $\mathcal{S}$ and $\alpha$.
\end{enumerate}
We show that (ii) implies (iii). By (*) and (**) we conclude that since $\mathcal{S}  \land \neg tr_{\alpha^*}(\mathcal{S})$ is satisfiable then $K_{\mathcal{S}} \land \neg {tr_{\alpha^*_K}(K_{\mathcal{S}})}$ is satisfiable. Let $I$ be a model of the latter. Let $\sigma$ be such a substitution of the form $x_1 \rightarrow a_1, \ldots, x_n \rightarrow a_n$ with $\sigma(\alpha_K) = \alpha^*_K$.  Moreover, assume that $x_1 \rightarrow a'_1, \ldots, x_n \rightarrow a'_n$ is the substitution that transforms $\alpha_K$ into $\alpha^c_K$. Let $I'$ be an instance that coincides with $I$ except that $(a'_i)^{I'} = (a_i)^I$ for each $1 \leq i \leq n$. Then, $I'$ is a model of $K_{\mathcal{S}} \land \neg {tr_{\alpha^c_K}(K_{\mathcal{S}})}$.

For (iii) implies (ii), let  $x_1 \rightarrow a'_1, \ldots, x_n \rightarrow a'_n$ be the substitution $\sigma$ that transforms $\alpha_K$ into $\alpha^c_K$ and let $I'$ be a model of $K_{\mathcal{S}} \land \neg {tr_{\alpha^c_K}(K_{\mathcal{S}})}$. 
Now, let $\sigma^*$ be a substitution obtained from $\sigma$ as follows: $x_i \rightarrow a_i$ if $a_i$ appears in $\mathcal{S}$ or $\alpha$ and $(a_i)^{I'} = (a'_i)^{I'}$ and let $x_i \rightarrow c_i$ with $c_i$ is a fresh node in $\Gamma$ if there is no node $a_i$ appearing in $\mathcal{S}$ or $\alpha$ such that $(a_i)^{I'} = (a'_i)^{I'}$. Then, $I$ which coincides with $I'$ except that $(\sigma^*(x_i))^I=\sigma^*(x_i)$, i.e., nodes are interpreted as themselves, is a model of $K_{\mathcal{S}} \land \neg {tr_{\alpha^*_K}(K_{\mathcal{S}})}$. From the latter, (iii), and (ii), follows that $\sigma^*$ is the desired substitution that transforms $\alpha$ into a ground instance $\alpha*$ such that $\mathcal{S}  \land \neg tr_{\alpha^*}(\mathcal{S})$ is satisfiable.

 The above claim, shows that we can reduce the co-problem of static validation under updates for the fragment of SHACL$^+$ inputs over $\mathcal{ACLHOI(Q)}^{br}$ constructs to checking finite satisfiability of $K_{\mathcal{S}} \land \neg {tr_{\alpha^c_K}(K_{\mathcal{S}})}$ in $\mathcal{ACLHOI(Q)}^{br}$. Clearly, ${tr_{\alpha^c_K}(K_{\mathcal{S}})}$ may be a boolean combination of exponentially many KBs where each of them is of polynomial size. It was shown in \cite{AhmetajCOS17} that finite satisfiability of such KBs is in \textsc{NExpTime} for $\mathcal{ALCHOIQ}^{br}$ and \textsc{ExpTime} for $\mathsf{ALCHOI}^{br}$. The proof holds also for the extension of $\mathcal{ALCHOIQ}^{br}$ that allows shape properties in places of roles.

\end{proof}

\paragraph*{\textbf{Implementation and Experiments}} Given a shape graph $\mathcal{S}$ and its equisatisfiable FOL sentence $\varphi$, and a ground action $a$ that the regression approach models as the substitution of expressions of type $r(x,y)$ with another expression $\phi$, $\varphi_{\{a\}} = \varphi^{r\rightarrow j} \wedge \forall x, y.\ j(x,y) \iff  \phi$, where $j$ is a fresh relation and $\varphi^{r\rightarrow j}$ is the sentence obtained by substituting every occurrence of relation $r$ with $j$. These substitutions result in a syntactic increase of the TPTP sentence linear to the number of ground actions performed. For example, let $\alpha$ be composed of a single ground action $p \add q$, where $p$ and $q$ are property names. Sentence $\varphi_{\alpha}$ is then generated by 1) substituting every syntactic occurrence of expression $r(x,y)$ in $\varphi$, for any variable or constant $x$ and $y$, with $j(x,y)$, and 2) appending axiom $\forall x, y . \ j(x,y) \iff r(x,y) \vee q(x,y)$ to $\varphi$. 

\todo{rewrote the part below. check if correct}

Our implementation currently supports a complex action comprised of a sequence of basic actions belonging defined in Section \ref{sec:updates}. The first type of action is the addition or removal of a property $p$ from a path SHACL$^+$ path $E^*$; we denote these actions $(p \add E)$ and $(p \del E)$. 

\end{document}